\documentclass[a4paper,12pt]{article}

\usepackage[utf8]{inputenc}
\usepackage[T1]{fontenc}
\usepackage{amsmath}
\usepackage{amssymb}
\usepackage{amsthm}
\usepackage{amsfonts}
\usepackage{graphicx}
\usepackage[autolanguage,np]{numprint}
\usepackage{tikz}
\usepackage{xspace}
\usepackage{hyperref}
\usepackage{xcolor}
\usepackage{algorithm}
\usepackage{algorithmic}
\usepackage{mathtools}

\theoremstyle{definition}
\newtheorem{definition}{Definition}

\theoremstyle{plain}
\newtheorem{theorem}{Theorem}
\newtheorem{lemma}{Lemma}
\newtheorem{corollary}{Corollary}

\DeclarePairedDelimiter\floor{\lfloor}{\rfloor}

\newcommand{\naturals}{\mathbb{N}}

\newcommand{\reals}{\mathbb{R}}
\newcommand{\field}{\mathbb{F}}
\newcommand{\binarySet}{\{0,1\}}
\newcommand{\hypercube}{\binarySet^n}

\newcommand{\range}[2]{[#1..#2]}
\newcommand{\innerProduct}[2]{\langle #1, #2 \rangle}

\DeclareMathOperator{\gl}{GL}
\DeclareMathOperator{\slg}{SL}
\DeclareMathOperator{\E}{\mathbb{E}}
\DeclareMathOperator{\proba}{\mathbb{P}}

\begin{document}

\title{Affine OneMax\footnote{An extended two-page abstract of this
    work will appear in \emph{2021 Genetic and Evolutionary
      Computation Conference Companion (GECCO '21 Companion)}.
    \url{https://doi.org/10.1145/3449726.3459497}}}

\author{Arnaud Berny}
\maketitle

\begin{abstract}
  A new class of test functions for black box optimization is
  introduced. Affine OneMax (AOM) functions are defined as
  compositions of OneMax and invertible affine maps on bit vectors.
  The black box complexity of the class is upper bounded by a
  polynomial of large degree in the dimension. The proof relies on
  discrete Fourier analysis and the Kushilevitz-Mansour algorithm.
  Tunable complexity is achieved by expressing invertible linear maps
  as finite products of transvections. The black box complexity of
  sub-classes of AOM functions is studied. Finally, experimental
  results are given to illustrate the performance of search algorithms
  on AOM functions.
\end{abstract}


\section{Introduction}

Theoretical and empirical analyses of search algorithms in the context
of black box optimization often require test functions. On the
practical side, an algorithm is usually selected by its performance
across a collection of diverse test functions such as OneMax,
LeadingOnes, MaxSat, etc. In particular, NK landscapes
\cite{kauffman-1993} are a class of functions with tunable complexity.
In this model, the fitness of an $n$-dimensional bit vector is the sum
of $n$ partial functions, one per variable. Each partial function
depends on a variable and its $k$ neighbors. The number $k$ controls
the interaction graph hence the complexity of the fitness landscape.
NK landscapes have found many applications from theoretical biology to
combinatorial optimization. When used as test functions, their
flexibility comes at the price of a great number of parameters. Values
of partial functions are often sampled from normal or uniform
distributions, which generates highly irregular landscapes with
unknown maximum, even for small $k$.

Affine OneMax (AOM) functions are test functions defined as
compositions of OneMax and invertible affine maps on bit vectors. They
are integer-valued, have a known maximum, and their representations
only require a number of bits quadratic in the dimension. The idea of
composition of a fitness function and a linear or affine map has
already been explored in the context of evolutionary computation
\cite{liepins-al-1990,liepins-vose-1991,salcedo-sanz-2007}. The
problem addressed in this line of research is to identify a
representation of the search space, e.g.\ an affine map, able to
transform a deceptive function into an easy one for genetic
algorithms. In Sec.~\ref{sec:black-box-complexity}, we propose an
algorithm which learns such a representation for AOM functions.
Considering the identity as a linear map, it appears that functions of
increasing complexity can be obtained starting with the identity and
applying small perturbations to it. In the language of group theory,
those perturbations are called elementary (or special) transvections.
Sequences of elementary transvections are multiplied to obtain
arbitrary linear maps. The key parameter of resulting AOM functions is
the sequence length which is the analogue of parameter $k$ in NK
landscapes.

The theory of black box complexity \cite{droste-2006,doerr-2020}
usually involves classes of functions rather than single functions.
The black box complexity $B_{\mathcal{F}}$ of a class $\mathcal{F}$ of
functions is defined as
$\inf_{A\in\mathcal{A}} \sup_{f\in\mathcal{F}} \allowbreak \E(T(A,
f))$, where $\mathcal{A}$ is the set of randomized search algorithms
and $\E(T(A, f))$ the expected runtime of algorithm $A$ on function
$f$. We will only consider unrestricted randomized search algorithms
hence unrestricted black box complexity. The usual way to devise a
function class is to supplement a given function with modified
versions of it. For example, instead of OneMax alone, it is customary
to consider the function class comprising all compositions of OneMax
with linear permutations (permutation of variables) and translations.
Given the fact that linear permutations and translations are
invertible affine maps, it seems natural to consider all invertible
affine maps. In this paper, we prove that the black box complexity of
AOM functions is upper bounded by $O(n^{10}\log^2 n)$. The proof uses
the Kushilevitz-Mansour algorithm \cite{kushilevitz-mansour-1993}
which is an application of discrete Fourier (or Walsh) analysis to the
approximation of Boolean functions. Fourier analysis has been
extensively used in the analysis of genetic algorithms
\cite{goldberg-1989-a} or fitness landscapes
\cite{stadler,heckendorn-al-1999-a,heckendorn-al-1999-b}. It also
plays an important role in the construction of a search matrix in the
analysis \cite{bshouty-2009} of the coin weighing problem
\cite{erdos-renyi-1963} which is related to OneMax.

The paper is organized as follows. Sec.~\ref{sec:general-case}
introduces general AOM functions and gives their basic properties.
Sec.~\ref{sec:tunable-complexity} introduces elementary transvections and
their products. Sec.~\ref{sec:black-box-complexity} addresses the
black box complexity of AOM functions. In Sec.~\ref{sec:experiments},
we report results of experiments involving search algorithms on random
instances. Sec.~\ref{sec:conclusion} concludes the paper.

\section{General case}
\label{sec:general-case}

The OneMax function $\ell : \hypercube \rightarrow \reals$ is defined
by $\ell(x) = \sum_{i=1}^n x_i$, where the $x_i$'s are seen as real
numbers. OneMax takes $n+1$ values $0, 1, \ldots, n$ and, for all
$k\in\range{0}{n}$, $|\ell^{-1}(k)| = \binom{n}{k}$. It reaches its
maximum $n$ only at $1^n = (1, 1, \ldots, 1)$.

From now on, the set $\binarySet$ is seen as the finite field
$\field_2$ ($1+1\equiv 0 \pmod{2}$) and $\hypercube$ as a linear space
over $\field_2$. The canonical basis of $\hypercube$ is denoted by
$(e_i)$, where $i\in\range{1}{n}$. An affine map
$\hypercube \rightarrow \hypercube$ is defined by $x \mapsto Mx+b$,
where $M$ is a $n\times n$ bit matrix and $b$ a $n\times 1$ bit
vector. Sums and products in $Mx+b$ are computed in $\field_2$ not in
$\reals$.

\paragraph{Definition and basic properties}

An affine OneMax function $f$ is the composition of OneMax and an
invertible affine map. More precisely, it is defined by
$f(x) = \ell(Mx+b)$, where $M\in \gl(n, \field_2)$, that is $M$ is an
invertible matrix. The invertibility of $M$ is important because it
ensures that $f$ shares the properties of OneMax outlined above.
Firstly, it takes $n+1$ values $0, 1, \ldots, n$ and, for all
$k\in\range{0}{n}$, $|f^{-1}(k)| = \binom{n}{k}$. Secondly, it reaches
its maximum $n$ only once at $x^*$ such that $Mx^*+b = 1^n$ or
$x^* = M^{-1}(1^n + b)$. The set of all AOM functions will be denoted
by $\mathcal{F}$. By construction, $\mathcal{F}$ is closed under
invertible affine maps but
\begin{theorem}
  $\mathcal{F}$ is not closed under permutations.
\end{theorem}

\begin{proof}
  Let $f = \ell \in \mathcal{F}$ and $\pi$ the (nonlinear) permutation
  of $\hypercube$ which exchanges $0$ and $e_1$ and leaves other
  vectors unchanged. We will prove by reductio ad absurdum that
  $f \circ \pi$ does not belong to $\mathcal{F}$. Suppose that
  $g = f \circ \pi \in \mathcal{F}$. Then, there exists
  $(M, b) \in \gl(n, \field_2) \times \hypercube$ such that, for all
  $x\in\hypercube$, $g(x) = \ell(Mx+b)$. We will need $g(0)$,
  $g(e_1)$, $g(e_2)$, and $g(e_3)$. We have
  $g(0) = (\ell \circ \pi)(0) = \ell(e_1) = 1$ and $g(0) = \ell(b)$.
  Then, there exists $i_1 \in \range{1}{n}$ such that $b = e_{i_1}$.
  We have $g(e_1) = (\ell \circ \pi)(e_1) = \ell(0) = 0$ and
  $g(e_1) = \ell(Me_1+b)$. Then,
  $Me_1+b = 0 \Leftrightarrow Me_1 = b = e_{i_1}$. We have
  $g(e_2) = (\ell \circ \pi)(e_2) = \ell(e_2) = 1$ and
  $g(e_2) = \ell(Me_2+b)$. Then, there exists $i_2 \in \range{1}{n}$
  such that
  $Me_2+b = e_{i_2} \Leftrightarrow Me_2 = e_{i_1} + e_{i_2}$.
  Similarly, there exists $i_3 \in \range{1}{n}$ such that
  $Me_3 = e_{i_1} + e_{i_3}$. To reach a contradiction, we consider
  $x = e_2 + e_3$. We have $g(x) = (\ell \circ \pi)(x) = \ell(x) = 2$
  and
  $g(x) = \ell(Mx+b) = \ell(Me_2 + Me_3 + b) = \ell(e_{i_1} + e_{i_2}
  + e_{i_3})$ which is necessarily odd. We conclude that $f \circ \pi$
  does not belong to $\mathcal{F}$.
\end{proof}

As a consequence, by a NFL theorem \cite{schumacher-al-2001}, the
average performance of an algorithm over all AOM functions depends on
the algorithm. Finally, we state a property relating AOM functions and
linear permutations. Let
$\Phi : \gl(n, \field_2) \times \hypercube \to \mathcal{F}$ be the map
defined by $\Phi(M, b) = f$, where $f(x) = \ell(Mx+b)$. Then, for all
$(M, b)$ and all linear permutations $\pi$,
$\Phi(\pi M, \pi b) = \Phi(M, b)$. This is a direct consequence of the
fact that OneMax itself is invariant under linear permutations, that
is, for all linear permutations $\pi$, $\ell \circ \pi = \ell$. Given
an AOM function $\Phi(M, b)$, an algorithm can only learn $M$ and $b$
up to a linear permutation.

To sample a random AOM function, one has to sample an invertible
matrix and a vector. An invertible matrix can be sampled by means of
rejection sampling. The success probabiblity $p_n$ is equal to
$|\gl(n, \field_2)| / 2^{n^2}$, where
$|\gl(n, \field_2)| = (2^n - 2^0)(2^n - 2^1) \cdots (2^n - 2^{n-1})$
and $2^{n^2}$ is the number of $n\times n$ matrices. By a convexity
argument,
$\ln p_n = \sum_{k=1}^n \ln(1-2^{-k}) \ge -2\ln 2 (1 - 2^{-n}) \ge
-2\ln 2$. Thus, the expected number of trials to sample an invertible
matrix is bounded from above by 4.

\paragraph{Fourier analysis}

Let $\mathcal{E}$ denote the space of pseudo-Boolean functions. For
all functions $f, g\in\mathcal{E}$, define their inner product by
$\innerProduct{f}{g} = 2^{-n} \sum_x f(x) \cdot g(x)$, where the sum
ranges over $\hypercube$. It will be useful to interpret inner
products as expectations. More precisely,
$\innerProduct{f}{g} = \E(fg) = \E(f(X)g(X))$, where $X$ is a random
vector with uniform distribution on $\hypercube$. For all
$x,u\in\hypercube$, let $\chi_u(x) = (-1)^{x\cdot u}$, where
$x\cdot u = \sum_i x_iu_i$. The $\chi_u$'s define an orthonormal basis
of $\mathcal{E}$. For all functions $f\in\mathcal{E}$,
$f = \sum_u \lambda_u \chi_u$, where
$\lambda_u = \innerProduct{f}{\chi_u}$. The set of coefficients
$\lambda_u$ is called the spectrum of $f$. The coefficient $\lambda_0$
is equal to the average value of $f$. A function $f$ is $t$-sparse if
it has at most $t$ nonzero coefficients. The Fourier transform
$\widehat{f}$ of $f$ is defined by $\widehat{f}(u) = \lambda_u$. The
support of $\widehat{f}$ is called the set of feature vectors of $f$.

The computation of the spectrum of OneMax gives the nonzero
coefficients $\widehat{\ell}(0) = n/2$ and
$\widehat{\ell}(e_i) = -1/2$. If we discard $\widehat{\ell}(0)$, we
can consider that it is made of $n$ nonzero coefficients of equal
amplitude, one for each basis vector. In order to compute the spectrum
of an AOM function, we use the known properties of Fourier transform
relative to linear maps and translations. For all $f\in\mathcal{E}$
and $M\in \gl(n, \field_2)$, if $g = f\circ M$ then
$\widehat{g} = \widehat{f} \circ (M^{-1})^\intercal$. For all
$b\in\hypercube$, let $T_b : \hypercube \rightarrow \hypercube$ be the
translation defined by $T_b(x) = x + b$. If $g=f\circ T_b$ then
$\widehat{g}(u) = (-1)^{u\cdot b} \widehat{f}(u)$.

Now, let $f = \Phi(M, b)$ be an AOM function. Then,
$f = (\ell \circ T_b) \circ M$,
$\widehat{f} = \widehat{(\ell \circ T_b)} \circ (M^{-1})^\intercal$,
and
$\widehat{f}(u) = \widehat{(\ell \circ T_b)}(v) = (-1)^{v\cdot b}
\widehat{\ell}(v)$, where $v = (M^{-1})^\intercal u$. Combining the
last result with the spectrum of $\ell$ and solving for $u$ in the
equations $(M^{-1})^\intercal u = e_i$, we obtain the nonzero
coefficients $\widehat{f}(0) = n/2$ and
$\widehat{f}(M^\intercal e_i) = -(-1)^{b_i}/2$. The spectrum of AOM
functions is similar to that of OneMax (same number of nonzero
coefficients and same amplitudes). There is an additional phase factor
$(-1)^{b_i}$ but the most important difference is that a nonzero
feature vector can now be any nonzero vector instead of just one of
the basis vectors. More precisely, the nonzero feature vectors of $f$
are the rows of $M$. Given that $M$ is invertible, they can form any
basis of $\hypercube$. As a consequence, AOM functions are not bounded
in the sense that if $u$ is a nonzero feature vector of an AOM
function then $\ell(u)$ can take any value in $\range{1}{n}$.

\section{Tunable complexity}
\label{sec:tunable-complexity}

We would like to sort AOM functions from the easiest to the hardest
ones to maximize or at least provide them with some structure. We will
achieve this goal using transvections. A transvection is an invertible
linear map $\hypercube \rightarrow \hypercube$ defined by
$x \mapsto x + h(x) a$, where $h$ is a non constant linear form and
$a\neq 0$ a vector such that $h(a) = 0$. Transvections generate the
special linear group $\slg(n, \field_2)$ which is represented by the
group of matrices of determinant 1. Since being $=1$ is the same as
being $\neq 0$ in $\field_2$, the special linear group
$\slg(n, \field_2)$ is also the general linear group
$\gl(n, \field_2)$ of invertible linear maps. In the general case, the
inverse of $x \mapsto x + h(x) a$ is $x \mapsto x - h(x) a$. In the
case of $\field_2$, a transvection is equal to its own inverse.

We concentrate on particular transvections called elementary
transvections. For all $i,j\in\range{1}{n}$, with $i\neq j$, let
$\tau_{ij}$ denote an elementary transvection. For all
$x\in\hypercube$, $\tau_{ij}x$ is obtained from $x$ by adding $x_j$ to
$x_i$ or $x_i \leftarrow x_i + x_j$, leaving other bits unchanged. At
most one bit is changed. It is clear that $\tau_{ij}$ is a linear map.
Its matrix is defined by $I_n + B_{ij}$, where $I_n$ is the
$n\times n$ identity matrix and $B_{ij}$ is the matrix whose $(i, j)$
entry is 1 and other entries are zero. Throughout the remainder of
this paper, we will write transvection to mean elementary
transvection. For all disjoint pairs $\{i, j\}$ and $\{k, l\}$, the
transvections $\tau_{ij}$ and $\tau_{kl}$ commute, that is,
$\tau_{ij} \tau_{kl} = \tau_{kl} \tau_{ij}$. Such transvections will
be said disjoint. We have already seen that $\tau_{ij}^2 = I_n$. For
all sets $\{i, j, k\}$, $\tau_{ij}$ and $\tau_{ik}$ (same destination)
commute, $\tau_{ik}$ and $\tau_{jk}$ (same source) commute, but
$\tau_{ij}$ and $\tau_{jk}$ do not.

Having defined transvections, we consider finite products of
transvections, gradually increasing the requirements, and the
corresponding classes of AOM functions. Let $t$ be a positive integer.
Let $G_t$ be the set of all products of $t$ transvections. Let $G_t^c$
be the set of all products of $t$ (pairwise) commuting transvections.
Let $M = \prod_{k=1}^t \tau_{i_kj_k}$, $I$ be the set of destination
indices, and $J$ be the set of source indices. The transvections are
commuting if and only if the sets $I$ and $J$ are disjoint.
Consequently, the set $G_t^c$ is defined (non empty) for all
$t\in\range{1}{\floor{n^2/4}}$. Let $G_t^\delta$ be the set of all
products of $t$ commuting transvections where each source index
appears at most once. In this case, for all products, we can define a
function $\delta$ from sources $J$ to destinations $I$. Let
$G_t^\sigma$ be the set of all products of $t$ commuting transvections
where each destination index appears at most once. In this case, for
all products, we can define a function function $\sigma$ from
destinations $I$ to sources $J$. Both sets $G_t^\delta$ and
$G_t^\sigma$ are defined for all $t\in\range{1}{(n-1)}$. Finally, for
all $t\in\range{1}{\floor{n/2}}$, let $G_t^d$ be the set of all
products of $t$ disjoint transvections. We have the inclusions
$G_t^d \subset G_t^\delta, G_t^\sigma \subset G_t^c \subset G_t
\subset \gl(n, \field_2)$.

We define function classes corresponding to sets of transvection
products. Let $\mathcal{F}_t$ be the set of functions
$x\mapsto \ell(Mx+b)$, where $M\in G_t$ and $b\in\hypercube$. Classes
$\mathcal{F}_t^c$, $\mathcal{F}_t^\sigma$, $\mathcal{F}_t^\delta$, and
$\mathcal{F}_t^d$ are defined similarly. We have the inclusions
$\mathcal{F}_t^d \subset \mathcal{F}_t^\delta, \mathcal{F}_t^\sigma
\subset \mathcal{F}_t^c \subset \mathcal{F}_t \subset \mathcal{F}$.
Functions in classes $\mathcal{F}_t$ will be refered to as
transvection sequence AOM (TS-AOM) functions.

As in the general case, we need to sample random TS-AOM functions in
$\mathcal{F}_t$. We propose to sample uniformly distributed random
sequences of $t$ transvections. However, the distribution of their
products is not uniform on $G_t$ and the distribution of resulting
functions is not uniform on $\mathcal{F}_t$.

\section{Black box complexity}
\label{sec:black-box-complexity}

Let us first study the black box complexity of general AOM functions.
We will give an upper bound of $B_{\mathcal{F}}$ by means of a
randomized search algorithm which efficiently learns and maximizes AOM
functions. Sparsity is the most important property of AOM functions.
It reminds us of Boolean functions which can be efficiently
approximated by sparse functions using the Kushilevitz-Mansour (KM)
algorithm \cite{kushilevitz-mansour-1993}. Indeed, we will prove that
the KM algorithm can efficiently learn the spectrum of AOM functions.
In the theory of fitness landscapes, most algorithms apply to bounded
functions such as NK landscapes \cite{choi-al-2008}. However, as
explained at the end of Sec.~\ref{sec:general-case}, AOM functions are
not bounded, hence the interest of the KM algorithm.

AOM functions are not Boolean functions so we need to transform them
before applying the KM algorithm.
\begin{lemma}
  \label{lemma:sparsity}
  Let $\mathcal{G} = \{ 2f/n - 1 : f\in\mathcal{F} \}$. If
  $g\in\mathcal{G}$ then $g$ has the following properties:
  \begin{enumerate}
  \item For all $x\in\hypercube$, $|g(x)|\le 1$.
  \item $g$ is $n$-sparse.
  \item For all $u\in\hypercube$, $\widehat{g}(u) = 0$ or
    $|\widehat{g}(u)| = 1/n$.
  \item $\E(g^2) = 1/n$.
  \end{enumerate}
\end{lemma}

We state the key definition and lemma from the work of Kushilevitz and
Mansour needed to understand their algorithm.

\begin{definition}
  For all $g\in\mathcal{G}$, all $k\in\range{1}{n-1}$, and all
  $u\in\binarySet^k$, the function
  $g_u : \binarySet^{n-k} \rightarrow \reals$ is defined by
  $g_u(x) = \sum_{v\in\binarySet^{n-k}} \widehat{g}(uv)\chi_v(x)$,
  where $uv$ is the concatenation of $u$ and $v$.
\end{definition}

\begin{lemma}
  For all $g\in\mathcal{G}$ and all $k\in\range{1}{n-1}$, the
  following properties are satisfied:
  \begin{enumerate}
  \item For all $u\in\binarySet^k$,
    $\E(g_u^2) = \sum_{v\in\binarySet^{n-k}} (\widehat{g}(uv))^2$.
  \item $\sum_{u\in\binarySet^k} \E(g_u^2) = \allowbreak \E(g^2)$.
  \item For all $u\in\binarySet^k$ and all $x\in\binarySet^{n-k}$,
    $g_u(x) = \E(g(Yx)\cdot \chi_u(Y))$, where $Y$ is a random vector
    with uniform distribution on $\binarySet^k$, and $|g_u(x)|\le 1$.
  \end{enumerate}
\end{lemma}

\begin{algorithm}
  \caption{Kushilevitz-Mansour algorithm}
  \label{algo:km}
  \begin{algorithmic}[1]
    \STATE $\text{KM}(g, m_1, m_2, u)$:
    \STATE $k \leftarrow \text{length}(u)$
    \FORALL{$i\in\range{1}{m_1}$}
    \STATE sample $x_i\in\binarySet^{n-k}$
    \FORALL{$j\in\range{1}{m_2}$}
    \STATE sample $y_{i,j}\in\binarySet^k$
    \ENDFOR
    \STATE $A_i \leftarrow (1 / m_2) \sum_{j=1}^{m_2} g(y_{i,j}x_i) \chi_u(y_{i,j})$
    \ENDFOR
    \STATE $B_u \leftarrow (1 / m_1) \sum_{i=1}^{m_1} A_i^2$
    \IF{$B_u > 1 / (2n^2)$}
    \IF{$k = n$}
    \RETURN $\{u\}$
    \ELSE
    \STATE $U_0\leftarrow\text{KM}(g, m_1, m_2, u0)$
    \STATE $U_1\leftarrow\text{KM}(g, m_1, m_2, u1)$
    \RETURN $U_0 \cup U_1$
    \ENDIF
    \ELSE
    \RETURN $\emptyset$
    \ENDIF
  \end{algorithmic}
\end{algorithm}

The Kushilevitz-Mansour algorithm (see Alg.~\ref{algo:km}) follows a
depth-first search of the complete binary tree of depth $n$. In this
context, $u\in\binarySet^k$ can be seen as a prefix or a path from the
root to an internal node of the tree. The expectation $\E(g_u^2)$
simply counts the number of leaves below $u$ and labelled by feature
vectors. If it is lower than $1/n^2$ then there is no need to expand
the node. However, expectations such as $\E(g_u^2)$ cannot be computed
efficiently so that a random approximation is required. We will use
Hoeffding's inequality \cite{hoeffding-1963} to control the
concentration of sample averages around their expectations.
\begin{lemma}[Hoeffding]
  Let $Z_1$, $Z_2$, ..., $Z_m$ be independent random variables such
  that, for all $i\in\range{1}{m}$, $a_i \le Z_i \le b_i$,
  $c_i = b_i - a_i$, and $\E(Z_i) = \mu$. Let
  $\overline{Z} = (1 / m) \sum_{i=1}^m Z_i$. Then, for all positive
  real numbers $d$,
  $\proba(|\overline{Z} - \mu| \ge d) \le 2 \exp(-2m^2d^2 /
  \sum_{i=1}^m c_i^2)$.
\end{lemma}

We state and prove a series of lemmas leading to the main theorem and
its corollary. This is adapted from the work of Kushilevitz and
Mansour to the case of AOM functions instead of Boolean functions.

\begin{lemma}
  \label{lemma:B_prime}
  Let $m_1\in\naturals^*$ and $X_1$, $X_2$, ..., $X_{m_1}$ be
  independent random vectors with uniform distribution on
  $\binarySet^{n-k}$. Let
  $B_u' = (1 / m_1) \sum_{i=1}^{m_1} g_u^2(X_i)$. Then
  \begin{equation*}
    \proba(|B_u' - \E(g_u^2)| \ge 1 / (4n^2)) \le 2 \exp(-m_1 /
    (8n^4)) \,.
  \end{equation*}
\end{lemma}

\begin{proof}
  Apply Hoeffding's inequality with $m = m_1$, $Z_i = g_u^2(X_i)$,
  $\mu = \E(g_u^2)$, $a_i = 0$, $b_i = 1$, $c_i = 1$, and
  $d = 1 / (4n^2)$.
\end{proof}

\begin{lemma}
  \label{lemma:A_i}
  Let $m_1, m_2\in\naturals^*$ and, for all $i\in\range{1}{m_1}$, let
  $Y_{i,1}$, $Y_{i,2}$, ..., $Y_{i,m_2}$ be independent random vectors
  with uniform distribution on $\binarySet^k$. For all
  $x\in\binarySet^{n-k}$, let
  $A_i(x) = (1 / m_2) \sum_{j=1}^{m_2} g(Y_{i,j}x) \cdot
  \chi_u(Y_{i,j})$. Then
  \begin{equation*}
    \proba(|A_i(x) - g_u(x)| \ge 1 / (8n^2)) \le 2 \exp(-m_2 /
    (128n^4)) \,.
  \end{equation*}
\end{lemma}

\begin{proof}
  We apply Hoeffding's inequality with $m = m_2$,
  $Z_j = g(Y_{i,j}x) \cdot \chi_u(Y_{i,j})$, $\mu = g_u(x)$,
  $a_i = -1$, $b_i = 1$, $c_i = 2$, and $d = 1 / (8n^2)$.
\end{proof}

\begin{lemma}
  Let $B_u = (1 / m_1) \sum_{i=1}^{m_1} A_i^2(X_i)$. If, for all
  $i\in\range{1}{n}$ and all $x\in\binarySet^{n-k}$,
  $|A_i(x) - g_u(x)| < 1 / (8n^2)$ then $|B_u - B_u'| < 1 / (4n^2)$.
\end{lemma}

\begin{proof}
  \begin{align*}
    B_u-B_u' &= (1/m_1) \sum_{i=1}^{m_1} (A_i^2(X_i) - g_u^2(X_i)) \\
             &= (1/m_1) \sum_{i=1}^{m_1} (A_i(X_i) - g_u(X_i)) \times (A_i(X_i) + g_u(X_i))
  \end{align*}
  Since both $g_u$ and $A_i$ are bounded in $[-1, 1]$, then
  \begin{equation*}
    |B_u-B_u'| \le (2/m_1) \sum_{i=1}^{m_1} |A_i(X_i) - g_u(X_i)| <
    1/(4n^2) \,.
  \end{equation*}
\end{proof}

\begin{lemma}
  \label{lemma:C_i}
  Let $m_3\in\naturals^*$ and $h = 2f - n$. For all $i\in\range{1}{n}$
  and all $u\in\hypercube$, let $X_{i,1}$, $X_{i,2}$, ..., $X_{i,m_3}$
  be independent random vectors with uniform distribution on
  $\binarySet^n$ and
  $C_i = (1 / m_3) \sum_{j=1}^{m_3} h(X_{i,j}) \cdot \chi_u(X_{i,j})$.
  Then
  $\proba(|C_i - \widehat{h}(u)| \ge 1) \le 2 \exp(-m_3 / (2n^2))$.
\end{lemma}

\begin{proof}
  Apply Hoeffding's inequality with $m = m_3$,
  $Z_j = h(X_{i,j}) \cdot \chi_u(X_{i,j})$, $\mu = \widehat{h}(u)$,
  $a_i = -n$, $b_i = n$, $c_i = 2n$, and $d = 1$.
\end{proof}

\begin{theorem}
  Let $f\in\mathcal{F}$ and $\delta\in(0, 1)$. There is a randomized
  algorithm which exactly learns and maximizes $f$ with at most
  $m_1m_2n^2+m_3n$ evaluations and probability at least $1-\delta$,
  where $m_1 = \Theta\left(n^4 \log(n^2/\delta)\right)$,
  $m_2 = \Theta\left(n^4 \log(n^2m_1/\delta)\right)$,
  $m_3 = \Theta\left(n^2 \log(n/\delta)\right)$.
\end{theorem}

\begin{algorithm}
  \caption{Algorithm for $\mathcal{F}$}
  \label{algo:general}
  \begin{algorithmic}[1]
    \STATE $\text{Maximize}(f, m_1, m_2, m_3)$:
    \STATE $g\leftarrow 2f/n - 1$
    \STATE $U\leftarrow\text{KM}(g, m_1, m_2, \emptyset)$
    \STATE $h\leftarrow 2f - n$
    \FORALL{$i\in\range{1}{n}$}
    \FORALL{$j\in\range{1}{m_3}$}
    \STATE sample $x_{i,j}\in\binarySet^n$
    \ENDFOR
    \STATE $C_i \leftarrow (1 / m_3) \sum_{j=1}^{m_3} h(x_{i,j}) \chi_{u_i}(x_{i,j})$
    \IF{$C_i > 0$}
    \STATE $b_i\leftarrow 1$
    \ELSE
    \STATE $b_i\leftarrow 0$
    \ENDIF
    \ENDFOR
    \STATE $M\leftarrow (u_1, u_2, \ldots, u_n)^\intercal$
    \RETURN $M^{-1}(1^n + b)$
  \end{algorithmic}
\end{algorithm}

\begin{proof}
  Alg.~\ref{algo:general} first calls Alg.~\ref{algo:km} to determine
  all $n$ nonzero feature vectors of $f$, then estimates $b$ before
  returning the solution. We arbitrarily divide the error probability
  $\delta$ into $\delta / 2$ for KM and $\delta / 2$ for learning $b$.
  It is further divided into $\delta / (2n^2)$ per recursive call in
  KM and $\delta / (2n)$ per component of $b$. Finally,
  $\delta / (4n^2)$ is allocated to the computation of $B_u$ and
  $\delta / (4n^2m_1)$ to that of each $A_i$. With the probability
  bound in Lemma~\ref{lemma:B_prime}, the condition
  $2 \exp(-m_1 / (8n^4)) = \delta / (4n^2)$ gives
  $m_1 = \Theta\left(n^4 \log(n^2/\delta)\right)$. With the
  probability bound in Lemma~\ref{lemma:A_i}, the condition
  $2 \exp(-m_2 / (128n^4)) = \delta / (4n^2m_1)$ gives
  $m_2 = \Theta\left(n^4 \log(n^2m_1/\delta)\right)$. Finally, with
  the probability bound in Lemma~\ref{lemma:C_i}, the condition
  $2 \exp(-m_3 / (2n^2)) = \delta / (2n)$ gives
  $m_3 = \Theta\left(n^2 \log(n/\delta)\right)$. With probability at
  least $1-\delta$, every sample average is sufficiently concentrated
  around its mean. In this case, in every call to KM,
  $|B_u-\E(g_u^2)| \le |B_u-B_u'| + |B_u'-\E(g_u^2)| <
  1/(4n^2)+1/(4n^2) = 1/(2n^2)$, where $\E(g_u^2) = 0$ or
  $\E(g_u^2) \ge 1 / n^2$. Similarly, every inequality
  $|C_i - \widehat{h}(u)| < 1$ is satisfied, where
  $\widehat{h}(u) = -1$ or 1. As a consequence,
  Alg.~\ref{algo:general} correctly learns and maximizes $f$ with
  probability at least $1-\delta$. The number of recursive calls in KM
  is at most $n^2$ and each one of them requires $m_1m_2$ evaluations.
  Moreover, $b$ has $n$ components and each one of them requires $m_3$
  evaluations.
\end{proof}

\begin{corollary}
  \label{corollary:B_F}
  $B_{\mathcal{F}} = O(n^{10}\log^2 n)$.
\end{corollary}

\begin{proof}
  We repeatedly run Alg.~\ref{algo:general} until a maximal point is
  found. The expected total number of evaluations is at most
  $(m_1m_2n^2+m_3n) / (1-\delta)$. Arbitrarily setting $\delta = 1/2$
  leads to $m_1m_2 = \Theta(n^8 \cdot \log^2n)$ and the announced
  upper bound.
\end{proof}

\begin{theorem}
  $B_{\mathcal{F}} = \Omega(n/\log n)$.
\end{theorem}

\begin{proof}
  AOM functions take all $n+1$ values $0, 1, \ldots, n$. For each
  $x\in\hypercube$, there exists $f\in\mathcal{F}$ such that $x$ is
  its maximal point, which is equivalent to the fact that the set
  $\{(M, b) \in \gl(n, \field_2) \times \hypercube : Mx+b=1^n\}$ is
  not empty. For example, choose any $M\in\gl(n, \field_2)$ and set
  $b=Mx+1^n$. By Theorem~2 in \cite{droste-2006},
  $B_{\mathcal{F}} = \Omega(\log_n 2^n) = \Omega(n/\log n)$.
\end{proof}

Let us now study the black box complexity of TS-AOM functions with
small sequence length.
\begin{theorem}
  Let
  $\mathcal{F}_1^0 = \{ \ell \circ \tau : \tau\in G_1 \} \subset
  \mathcal{F}_1$. Then $B_{\mathcal{F}_1^0} = O(\log n)$.
\end{theorem}

\begin{proof}
  The proof of the upper bound relies on a deterministic algorithm
  (see Alg.~\ref{algo:F_1_0}). Let $f = \ell\circ\tau_{ij}$ be in
  $\mathcal{F}_1^0$. For all $k\in\range{1}{n}\setminus \{j\}$,
  $\tau_{ij}(e_k) = e_k$ and $\tau_{ij}(e_j) = e_i+e_j$. The algorithm
  first calls $\text{BinarySearch}$ with $J=\emptyset$ and $c=0$ to
  find the source index $j$. At each stage, the candidate set $K$ is
  split into $K'$ and $K''$ of almost equal size, that is, the
  distance between $|K'|$ and $|K''|$ is at most 1. Let
  $\varepsilon = \sum_{k\in K'} e_k$. If $j\in K'$ then
  $\tau_{ij}(\varepsilon) = \varepsilon +e_i$ else $\varepsilon$. If
  $f(\varepsilon) = \ell(\varepsilon)$ then $j\in K''$ else $j\in K'$.
  The algorithm then calls $\text{BinarySearch}$ with $J=\{j\}$ and
  $c=1$ to find the destination index $i$. The candidate set $K$ does
  not contain $j$. It is split as before. Let
  $\varepsilon = \sum_{k\in K'} e_k + e_j$. We have
  $\tau_{ij}(\varepsilon) = \varepsilon +e_i$ and
  $f(\varepsilon) = \ell(\varepsilon +e_i)$. If
  $f(\varepsilon) = \ell(\varepsilon) + 1$ then $i\in K''$ else
  $i\in K'$. As a consequence, Alg.~\ref{algo:F_1_0} maximizes
  functions in $\mathcal{F}_1^0$ with at most $2\log_2 n$ evaluations.
\end{proof}

\begin{algorithm}
  \caption{Algorithm for $\mathcal{F}_1^0$}
  \label{algo:F_1_0}
  \begin{algorithmic}[1]
    \STATE $A_1^0(f)$:
    \STATE $j\leftarrow\text{BinarySearch}(f, \range{1}{n}, \emptyset, 0)$
    \STATE $i\leftarrow\text{BinarySearch}(f, \range{1}{n}\setminus\{j\}, \{j\}, 1)$
    \RETURN $1^n + e_i$
  \end{algorithmic}
\end{algorithm}

\begin{algorithm}
  \caption{Binary search for $\mathcal{F}_1^0$}
  \begin{algorithmic}[1]
    \STATE $\text{BinarySearch}(f, K, J, c)$:
    \IF{$|K|=1$}
    \RETURN $k\in K$
    \ENDIF
    \STATE Split $K$ into $K'$ and $K''$
    \STATE $\varepsilon \leftarrow \sum_{k\in K'\cup J} e_k$
    \IF{$f(\varepsilon) = \ell(\varepsilon) + c$}
    \RETURN $\text{BinarySearch}(f, K'', J, c)$
    \ELSE
    \RETURN $\text{BinarySearch}(f, K', J, c)$
    \ENDIF
  \end{algorithmic}
\end{algorithm}

\begin{theorem}
  $B_{\mathcal{F}_1} = O(n)$.
\end{theorem}

\begin{proof}
  The upper bound is achieved by a deterministic algorithm (see
  Alg.~\ref{algo:F_1}). We will use the following property of OneMax.
  For all $x, y\in\hypercube$, the Hamming distance between $x$ and
  $y$ can be expressed as $\ell(x+y)$ and
  $\ell(x+y) \equiv \ell(x)-\ell(y) \pmod{2}$. Let
  $f = \ell \circ T_b \circ \tau_{ij}$ be in $\mathcal{F}_1$. For all
  $k\in\range{1}{n}\setminus \{j\}$, $\tau_{ij}(e_k) = e_k$ and
  $\tau_{ij}(e_j) = e_i+e_j$. The algorithm first determines $j$. We
  have
  $\alpha_0 = f(0) = (\ell \circ T_b \circ \tau_{ij})(0) = \ell(b)$
  and
  $\alpha_j = f(e_j) = (\ell \circ T_b \circ \tau_{ij})(e_j) = \ell(b
  + e_i + e_j)$. For all $k\neq j$,
  $\alpha_k = f(e_k) = (\ell \circ T_b \circ \tau_{ij})(e_k) = \ell(b
  + e_k)$. We have
  $\alpha_j - \alpha_0 \equiv \ell(e_i+e_j) \equiv 0 \pmod{2}$ and,
  for all $k\neq j$,
  $\alpha_k - \alpha_0 \equiv \ell(e_k) \equiv 1 \pmod{2}$. The
  algorithm then determines $b$. For all $k\neq j$, if $b_k = 0$ then
  $\alpha_k - \alpha_0 = \ell(b + e_k) - \ell(b) = 1$ else $-1$. Let
  $\varepsilon = \sum_{k\neq j} b_k e_k = b + b_je_j$. Using the fact
  that $\varepsilon$ is invariant under $\tau_{ij}$, we have
  $f(\varepsilon) = (\ell \circ T_b \circ \tau_{ij})(\varepsilon) =
  \ell(b + \varepsilon) = \ell(b_je_j) = b_j$. Finally, the algorithm
  determines $i$. We have
  $\beta_j = f(e_j + \varepsilon) = (\ell \circ T_b \circ
  \tau_{ij})(e_j + \varepsilon) = (\ell \circ T_b)(e_i + e_j +
  \varepsilon) = \ell(b + e_i + e_j + \varepsilon) = \ell(e_i +
  (1+b_j)e_j)$. Similarly, for all $k\neq j$,
  $\beta_{j,k} = \ell(e_i + (1+b_j)e_j + e_k)$. If $k=i$ then
  $\beta_{j,k} - \beta_j = -1$ else $1$. The total number of
  evaluations is equal to $1 + n + 2 + (n-1) = 2(n+1)$, which implies
  $B_{\mathcal{F}_1} = O(n)$.
\end{proof}

\begin{algorithm}
  \caption{Algorithm for $\mathcal{F}_1$}
  \label{algo:F_1}
  \begin{algorithmic}[1]
    \STATE $A_1(f)$:
    \STATE $\alpha_0 \leftarrow f(0)$
    \FORALL{$k\in\range{1}{n}$}
    \STATE $\alpha_k \leftarrow f(e_k)$
    \ENDFOR
    \STATE $j \in \{k : \alpha_k - \alpha_0 \equiv 0 \pmod{2}\}$ \label{line:j}
    \FORALL{$k\in\range{1}{n} \setminus \{j\}$}
    \IF{$\alpha_k - \alpha_0 = 1$}
    \STATE $b_k \leftarrow 0$
    \ELSE
    \STATE $b_k \leftarrow 1$
    \ENDIF
    \ENDFOR
    \STATE $\varepsilon \leftarrow \sum_{k\neq j} b_k e_k$
    \STATE $b_j \leftarrow f(\varepsilon)$
    \STATE $\beta_j \leftarrow f(e_j + \varepsilon)$
    \FORALL{$k\in\range{1}{n} \setminus \{j\}$}
    \STATE $\beta_{j,k} \leftarrow f(e_j + e_k + \varepsilon)$
    \ENDFOR
    \STATE $i \in \{k\neq j : \beta_{j,k} - \beta_j = -1\}$ \label{line:i}
    \RETURN $\tau_{ij}(1^n+b)$
  \end{algorithmic}
\end{algorithm}

\begin{theorem}
  For all positive integer $t$, if $t\le 5$ then
  $B_{\mathcal{F}_t} = O(n^{2t-1})$ else $O(n^{10}\log^2 n)$.
\end{theorem}

\begin{proof}
  Assume Alg.~\ref{algo:F_1} has been modified so as to return an
  arbitrary solution if $f\notin\mathcal{F}_1$. Assume $t>1$ and let
  $f\in\mathcal{F}_t$. Then $f$ can be written
  $\ell \circ T_b \circ M$, where $M = \Pi_{k=1}^t \tau_k$. Let
  $P = \Pi_{k=t}^2 \tau_k$. Since $MP = \tau_1$, we have
  $f\circ P\in\mathcal{F}_1$. Let $A_t$ be the algorithm which
  enumerates all products $P$ of $t-1$ transvections and applies $A_1$
  on each $f\circ P$ until a solution is found. Since the number of
  transvections is $n(n-1)$, we have
  $B_{\mathcal{F}_t} \le (n(n-1))^{t-1} \times B_{\mathcal{F}_1} =
  O(n^{2t-1})$. If $t>5$ then the general bound $O(n^{10}\log^2 n)$
  given by Corollary~\ref{corollary:B_F} is better than $O(n^{2t-1})$.
\end{proof}

\begin{theorem}
  For all positive integer $t \le n-1$,
  $B_{\mathcal{F}_t^\delta} \le n + t(\log_2(n-t) + 1)$ and
  $B_{\mathcal{F}_t^\delta} = O(n\log n)$.
\end{theorem}

\begin{proof}
  The bound is achieved by a deterministic algorithm (see
  Alg.~\ref{algo:F_t_delta}) similar to Alg.~\ref{algo:F_1}. Let
  $f = \ell \circ T_b \circ M$ be in $\mathcal{F}_t^\delta$, where
  $M = \prod_{j\in J} \tau_{\delta(j)j}$, $J$ is the set of source
  indices, $t = |J|$, and, for all $j\in J$, $\delta(j)$ is the
  destination index of transvection $\tau_{\delta(j)j}$. Recall that
  $J$ and $\delta(J)$ are disjoint. The algorithm first determines
  $J$. If $k\notin J$ then $e_k$ is invariant under $M$,
  $f(e_k) = \ell(b+e_k)$, and
  $f(e_k)-f(0) \equiv \ell(e_k) \equiv 1 \pmod{2}$. If $k\in J$ then
  $Me_k = e_k + e_{\delta(k)}$, $f(e_k) = \ell(b+e_k+e_{\delta(k)})$,
  and $f(e_k)-f(0) \equiv \ell(e_k+e_{\delta(k)}) \equiv 0 \pmod{2}$.
  Let $K = \range{1}{n} \setminus J$ be the candidate set for
  destination indices. The algorithm then determines $b_k$ for all
  $k\in K$ as Alg.~\ref{algo:F_1} does. Let
  $\varepsilon = \sum_{k\in K} b_k e_k$. Observe that $\varepsilon$ is
  invariant under $M$ and $\varepsilon+b = \sum_{k\in J} b_k e_k$. The
  algorithm then determines $b_j$ for all $j\in J$. Let $j\in J$ and
  $i = \delta(j)$. We have
  $\beta_0 = f(\varepsilon) = \ell(\varepsilon+b) = \sum_{k\in J} b_k
  = \sum_{k\in J\setminus\{j\}} b_k + b_j$. We also have
  $\beta_j = f(\varepsilon+e_j) = \ell(\varepsilon+b+e_j+e_i) =
  \sum_{k\in J\setminus\{j\}} b_k + (1-b_j) + 1$. If $b_j = 0$ then
  $\beta_j-\beta_0 = 2$ else $0$. The algorithm then determines
  $i = \delta(j)$ by means of binary search as in
  Alg.~\ref{algo:F_1_0}. We have
  $\beta_j = \ell(\varepsilon+b+e_j+e_i) = \ell(\varepsilon+b+e_j) +
  \ell(e_i)$. At each stage, the candidate set $L\subset K$ is split
  into $L'$ and $L''$ of almost equal size. Let
  $\lambda = \sum_{k\in L'} e_k$, which is invariant under $M$, and
  $\beta_{j,\lambda} = f(\varepsilon + e_j + \lambda)$. We have
  $\beta_{j,\lambda} = \ell(\varepsilon+b+e_j) + \ell(e_i+\lambda)$.
  If $i\in L'$ then $\beta_{j,\lambda} - \beta_j = \ell(\lambda) - 2$
  else $\ell(\lambda)$. The total number of evaluations is equal to
  $1 + n + 1 + t \times(1 + \log_2(n-t)) = n + t(\log_2(n-t) + 1) + 2
  \le n + n(\log_2 n + 1) + 2 = n\log_2 n + 2n + 2 = O(n\log n)$.
  Alg.~\ref{algo:F_t_delta} does not apply to $\mathcal{F}_t^c$ or
  $\mathcal{F}_t^\sigma$ since it cannot identify the set $J$ of
  source indices if the same source appears more than once in the
  transvection sequence.
\end{proof}

\begin{algorithm}
  \caption{Algorithm for $\mathcal{F}_t^\delta$}
  \label{algo:F_t_delta}
  \begin{algorithmic}[1]
    \STATE $A_t^\delta(f)$:
    \STATE $\alpha_0 \leftarrow f(0)$
    \FORALL{$k\in\range{1}{n}$}
    \STATE $\alpha_k \leftarrow f(e_k)$
    \ENDFOR
    \STATE $J = \{k : \alpha_k - \alpha_0 \equiv 0 \pmod{2}\}$
    \STATE $K = \range{1}{n} \setminus J$
    \FORALL{$k\in K$}
    \IF{$\alpha_k - \alpha_0 = 1$}
    \STATE $b_k \leftarrow 0$
    \ELSE
    \STATE $b_k \leftarrow 1$
    \ENDIF
    \ENDFOR
    \STATE $\varepsilon \leftarrow \sum_{k\in K} b_k e_k$
    \STATE $\beta_0\leftarrow f(\varepsilon)$
    \FORALL{$j\in J$}
    \STATE $\beta_j \leftarrow f(\varepsilon+e_j)$
    \IF{$\beta_j-\beta_0 = 2$}
    \STATE $b_j \leftarrow 0$
    \ELSE
    \STATE $b_j \leftarrow 1$
    \ENDIF
    \STATE $\delta(j)\leftarrow\text{BinarySearch}(f, \varepsilon, j, \beta_j, K)$
    \ENDFOR
    \RETURN $\prod_{j\in J}\tau_{\delta(j)j}(1^n+b)$
  \end{algorithmic}
\end{algorithm}

\begin{algorithm}
  \caption{Binary search for $\mathcal{F}_t^\delta$}
  \begin{algorithmic}[1]
    \STATE $\text{BinarySearch}(f, \varepsilon, j, \beta_j, L)$:
    \IF{$|L|=1$}
    \RETURN $k\in L$
    \ENDIF
    \STATE Split $L$ into $L'$ and $L''$
    \STATE $\lambda \leftarrow \sum_{k\in L'} e_k$
    \STATE $\beta_{j,\lambda} \leftarrow f(\varepsilon + e_j + \lambda)$
    \IF{$\beta_{j,\lambda} - \beta_j = \ell(\lambda)$}
    \RETURN $\text{BinarySearch}(f, \varepsilon, j, \beta_j, L'')$
    \ELSE
    \RETURN $\text{BinarySearch}(f, \varepsilon, j, \beta_j, L')$
    \ENDIF
  \end{algorithmic}
\end{algorithm}

\section{Experiments}
\label{sec:experiments}

We have applied the following search algorithms to random instances of
AOM and TS-AOM functions: random search (RS), random local search with
restart (RLS), hill climbing with restart (HC), simulated annealing
(SA) \cite{kirkpatrick-gelatt-vecchi-1983}, genetic algorithm (GA)
\cite{holland-75}, $(1+1)$ evolutionary algorithm (EA), $(10+1)$
evolutionary algorithm, population-based incremental learning (PBIL)
\cite{baluja-caruana-95}, mutual information maximization for input
clustering (MIMIC) \cite{isbell}, univariate marginal distribution
algorithm (UMDA) \cite{muhlenbein-97}, hierarchical Bayesian
optimization algorithm (HBOA) \cite{pelikan06:_hierar_bayes}, linkage
tree genetic algorithm (LTGA) \cite{thierens-2010}, parameter-less
population pyramid (P3) \cite{goldman-punch-2015}. All experiments
have been produced with the HNCO framework \cite{hnco}.

Fig.~\ref{fig:aom_benchmark} shows the fixed-budget performance of
search algorithms on the same AOM function which has been sampled as
described in Sec.~\ref{sec:general-case}. Its maximum is equal to
$n = 100$. Median function values are not greater than 73. Results
across algorithms are similar and would not significantly change if
algorithms were run on another random instance. It should be noted
that random search performs as well as other algorithms. These results
do not come as a surprise since a random invertible matrix with
uniform distribution is most likely very different from the identity.
As we will see in the following experiments, TS-AOM functions with
small sequence length are the easiest AOM functions.

\begin{figure}
  \centering
  \includegraphics[width=0.8\linewidth]{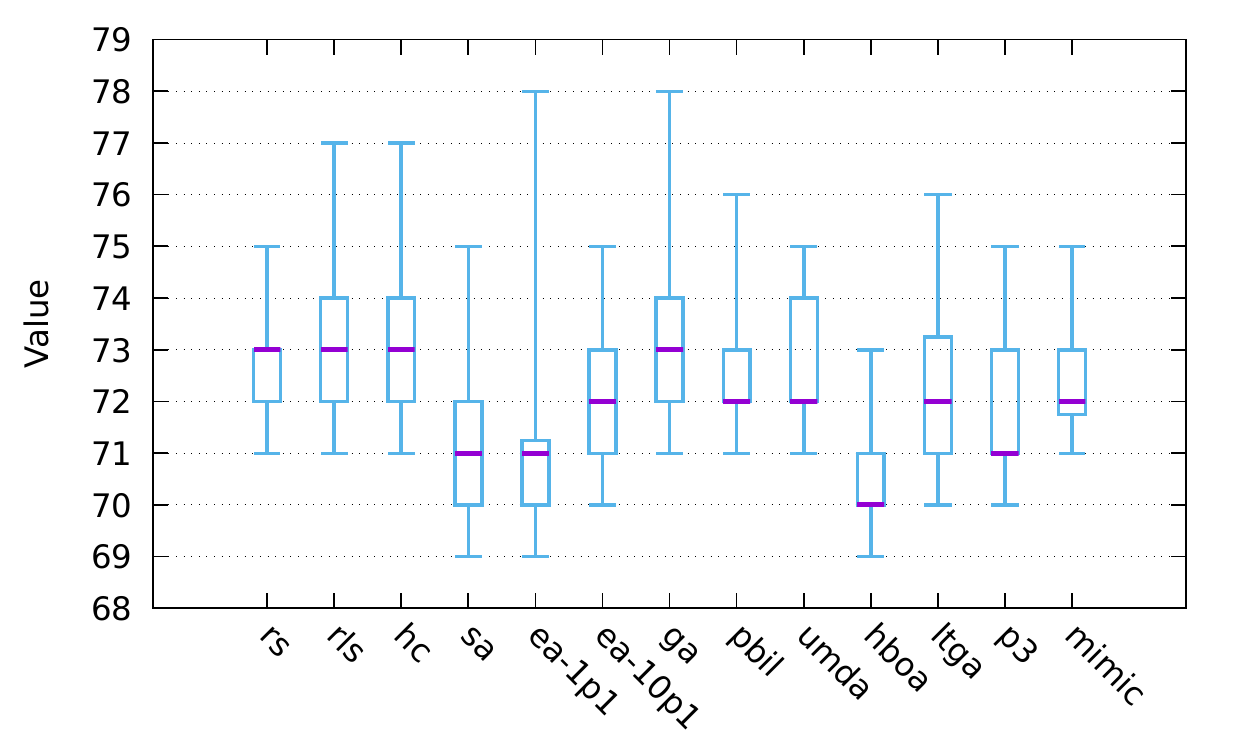}
  \caption{Fixed-budget performance of some search algorithms on the
    same AOM function ($n = 100$, $3\cdot 10^5$ evaluations, 20 runs).}
  \label{fig:aom_benchmark}
\end{figure}

\begin{figure}
  \centering
  \includegraphics[width=0.8\linewidth]{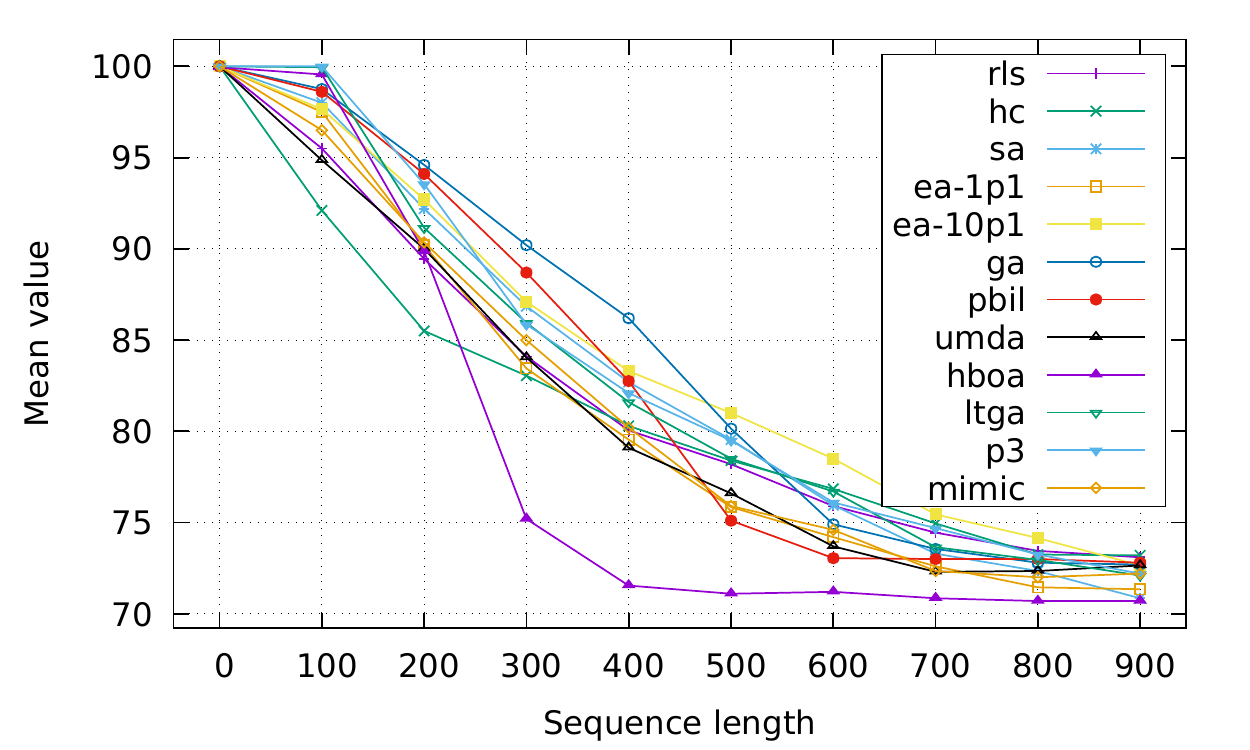}
  \caption{Mean fixed-budget performance of search algorithms on
    TS-AOM functions as a function of $t\in\range{0}{900}$ ($n = 100$,
    $3\cdot 10^5$ evaluations, 20 runs).}
  \label{fig:ts_aom_unconstrained_function_parameter_global}
\end{figure}

Fig.~\ref{fig:ts_aom_unconstrained_function_parameter_global} shows
the mean fixed-budget performance of search algorithms on TS-AOM
functions as a function of sequence length $t\in\range{0}{900}$. For
each run, a random TS-AOM function has been generated. The randomness
of the performance stems from the random generation of the function
and the stochastic behavior of algorithms. This will also be the case
for subsequent experiments. For all search algorithms, as $t$
increases, the mean value converges to the mean value observed in the
general case (see Fig.~\ref{fig:aom_benchmark}). This is consistent
with the theory of random walks on finite groups
\cite{saloff-coste-2004}. It should be noted that some results
\cite{hildebrand-1992} only apply to general transvections, as opposed
to elementary transvections used in this paper. Finite products of
random transvections can be seen as random walks on
$\gl(n, \field_2)$. As $t\rightarrow \infty$, their distributions
converge to the uniform distribution on $\gl(n, \field_2)$. For large
$t$, the distribution of random TS-AOM functions is close to the
uniform distribution on $\mathcal{F}$, which explains the observed
asymptotic behavior. For $t\in\range{20}{200}$, hill climbing is
significantly worse than other algorithms. For $t\in\range{200}{300}$,
the mean value of HBOA drops faster than that of other algorithms. For
$t\in\range{200}{460}$, the genetic algorithm outperforms other
algorithms. For $t\in\range{500}{800}$, $(10+1)$ EA outperforms other
algorithms.

\begin{figure}
  \centering
  \includegraphics[width=0.8\linewidth]{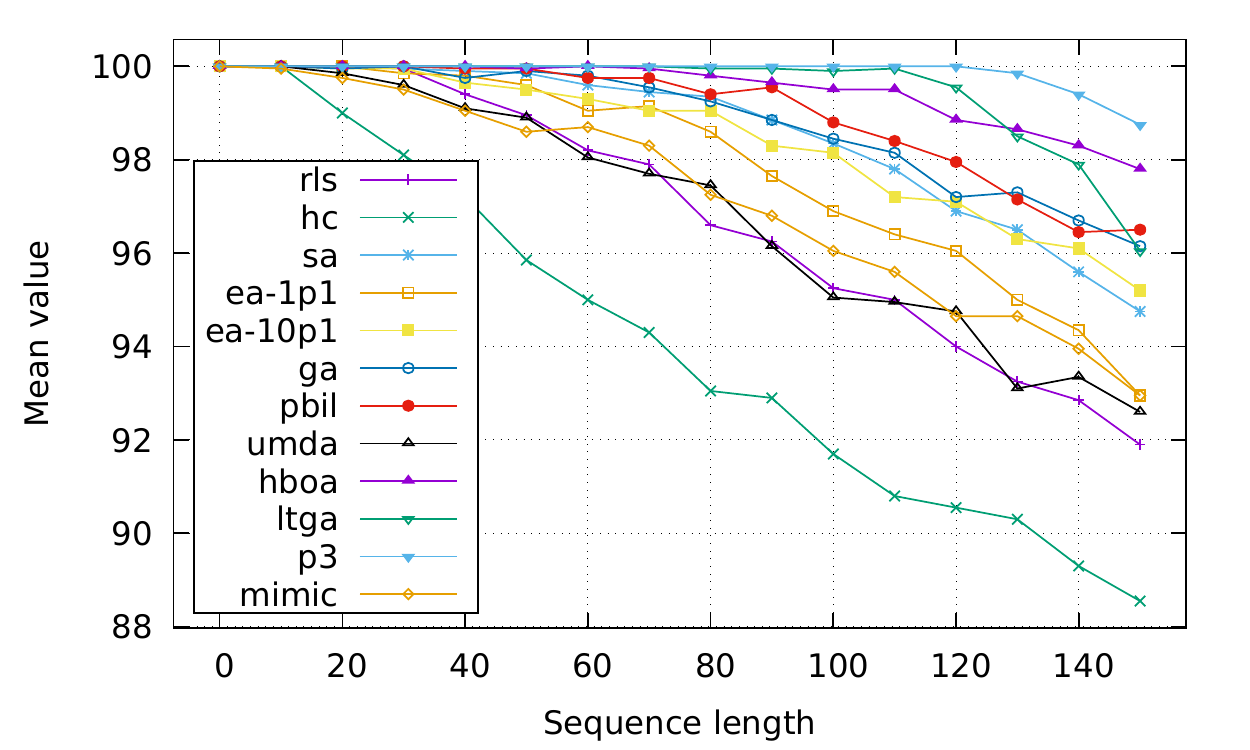}
  \caption{Mean fixed-budget performance of search algorithms on
    TS-AOM functions as a function of $t\in\range{0}{150}$ ($n = 100$,
    $3\cdot 10^5$ evaluations, 20 runs).}
  \label{fig:ts_aom_unconstrained_function_parameter_beginning}
\end{figure}

In Fig.~\ref{fig:ts_aom_unconstrained_function_parameter_beginning},
we turn our attention to TS-AOM functions with small sequence length,
that is those not too far from OneMax. The plot shows how search
algorithms resist an increasing number of perturbations (elementary
transvections). MIMIC is the first algorithm to fail to maximize an
instance (at $t=10$) whereas P3 is the last one (at $t=130$).

\begin{figure}
  \centering
  \includegraphics[width=0.8\linewidth]{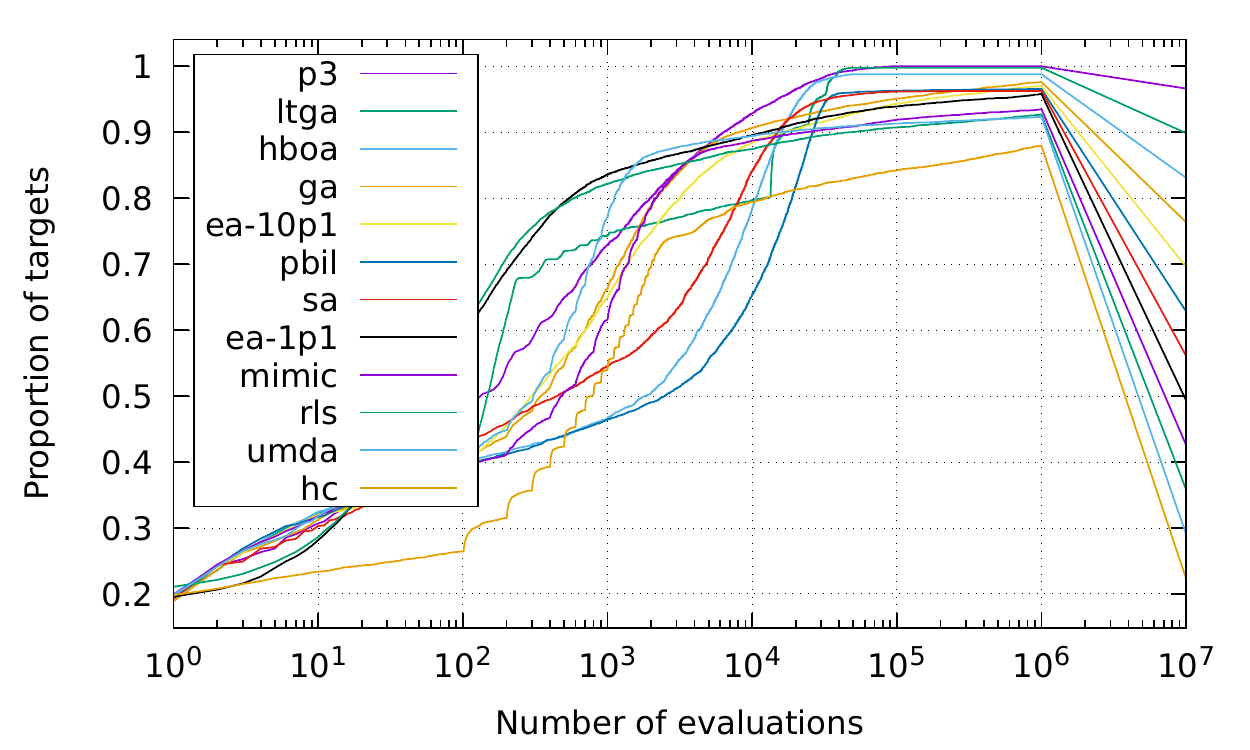}
  \caption{ECDF of search algorithms on 10 TS-AOM functions
    ($n = 100$, $t = 100$, $10^6$ evaluations, 20 run). Segments on
    the far right are helper lines to distinguish search algorithms.}
  \label{fig:ts_aom_unconstrained_ecdf}
\end{figure}

Fig.~\ref{fig:ts_aom_unconstrained_ecdf} shows the empirical
cumulative distribution functions
\cite{DBLP:journals/corr/HansenABTT16} of search algorithms on the
same set of 10 TS-AOM functions with $n= 100$ and $t = 100$.

Fig.~\ref{fig:aom_runtime} shows the mean runtime of $(1+1)$ EA on AOM
functions as a function of $n$. The results suggest that the expected
runtime of $(1+1)$ EA on AOM functions is exponential in $n$.
Fig.~\ref{fig:ts_aom_unconstrained_runtime} shows the mean runtime of
$(1+1)$ EA on TS-AOM functions with $n = 11$ as a function of $t$. As
$t\rightarrow \infty$, the runtime converges to an asymptotic value
which is the runtime observed in the general case (see
Fig.~\ref{fig:aom_runtime}). As already noted, this is consistent with
the theory of random walks on finite groups.

\begin{figure}
  \centering
  \includegraphics[width=0.8\linewidth]{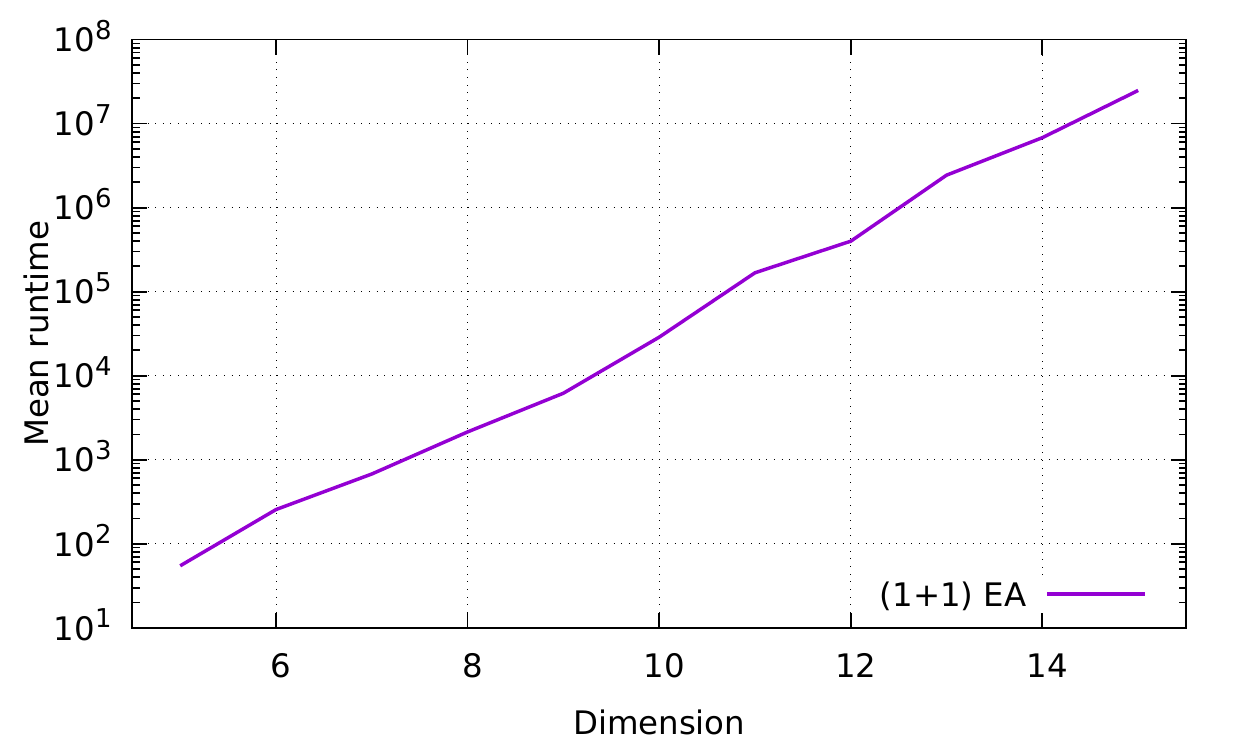}
  \caption{Mean runtime of $(1+1)$ EA on AOM functions as a function
    of $n$ (50 runs).}
  \label{fig:aom_runtime}
\end{figure}

\begin{figure}
  \centering
  \includegraphics[width=0.8\linewidth]{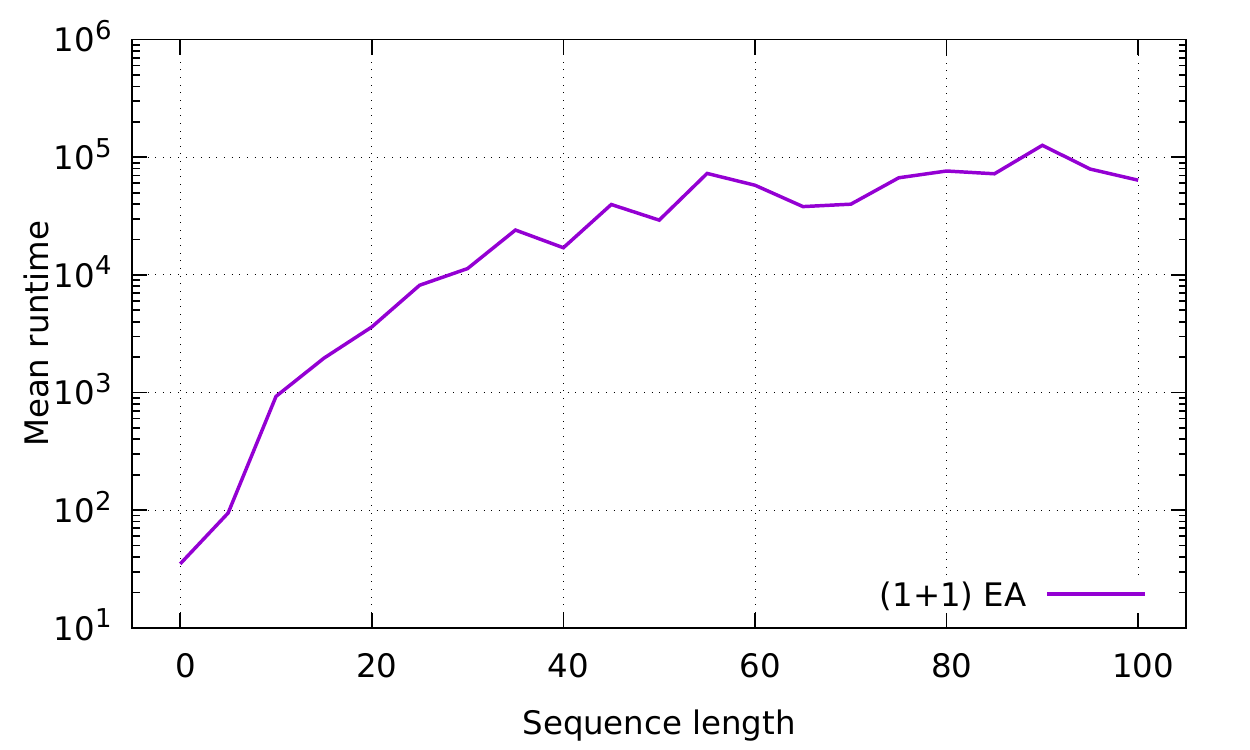}
  \caption{Mean runtime of $(1+1)$ EA on TS-AOM functions as a
    function of $t$ ($n = 11$, 100 runs).}
  \label{fig:ts_aom_unconstrained_runtime}
\end{figure}

Fig.~\ref{fig:ts_aom_commuting_noncommuting_runtime_t} shows the mean
runtime of $(1+1)$ EA on TS-AOM functions with commuting and non
commuting transvections and $n=9$ as a function of $t$. By
noncommuting transvections we mean that no two consecutive
transvections in the sequence commute. We acknowledge that such TS-AOM
functions are negligible in the set
$\mathcal{F}_t \setminus \mathcal{F}_t^c$. The maximum sequence length
is set to $(n^2-1)/4=20$, which is the maximum allowed for commuting
transvections. Some functions with commuting transvections are
significantly harder than others. On the contrary to commuting
transvections, the mean and the standard deviation of the runtime with
non commuting transvections do not exhibit any peak.

\begin{figure}
  \centering
  \includegraphics[width=0.8\linewidth]{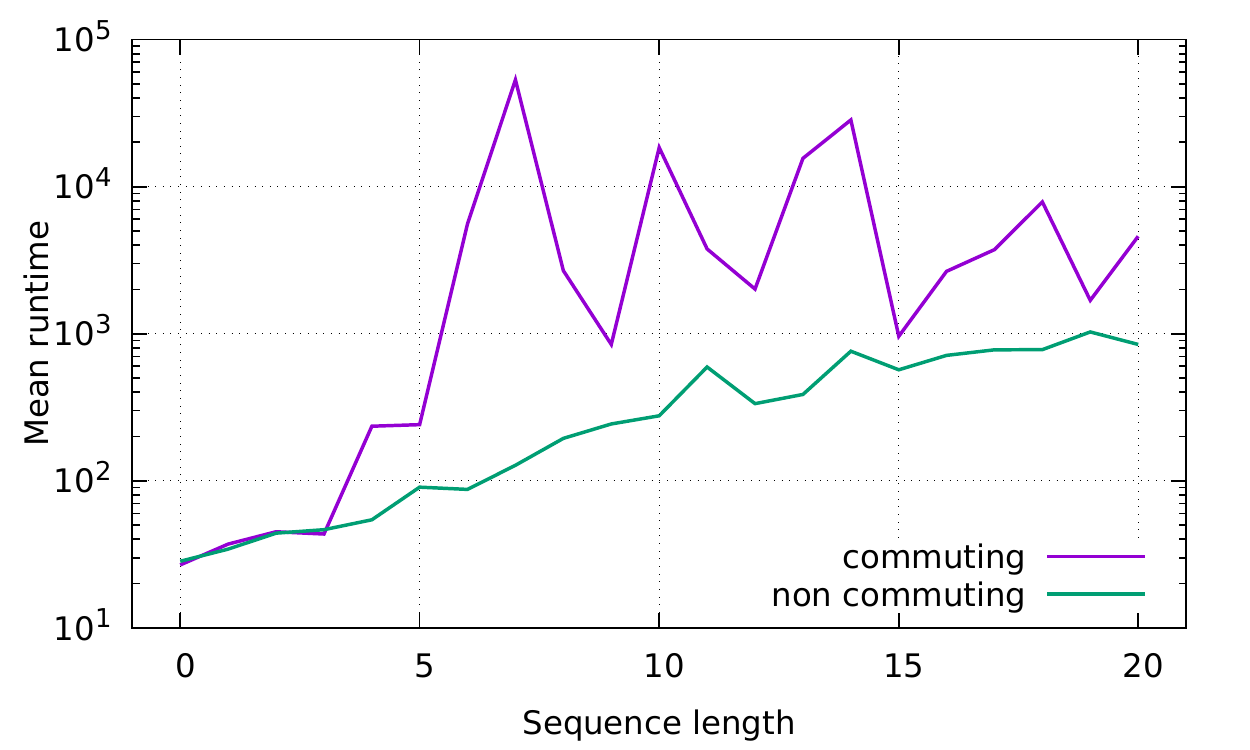}
  \caption{Mean runtime of $(1+1)$ EA on TS-AOM functions with
    commuting and noncommuting transvections as a function of $t$
    ($n=9$, 100 runs).}
  \label{fig:ts_aom_commuting_noncommuting_runtime_t}
\end{figure}

Fig.~\ref{fig:ts_aom_delta_sigma_runtime_t} shows the mean runtime of
$(1+1)$ EA on functions in $\mathcal{F}_t^\delta$ (unique destination)
and $\mathcal{F}_t^\sigma$ (unique source) with $n=9$ as a function of
$t$. The maximum sequence length is set to $n-1=8$, which is the
maximum allowed for both $\mathcal{F}_t^\delta$ and
$\mathcal{F}_t^\sigma$. There is a clear separation between the two
classes and $\mathcal{F}_t^\sigma$ appears to contain the hardest
functions.

\begin{figure}
  \centering
  \includegraphics[width=0.8\linewidth]{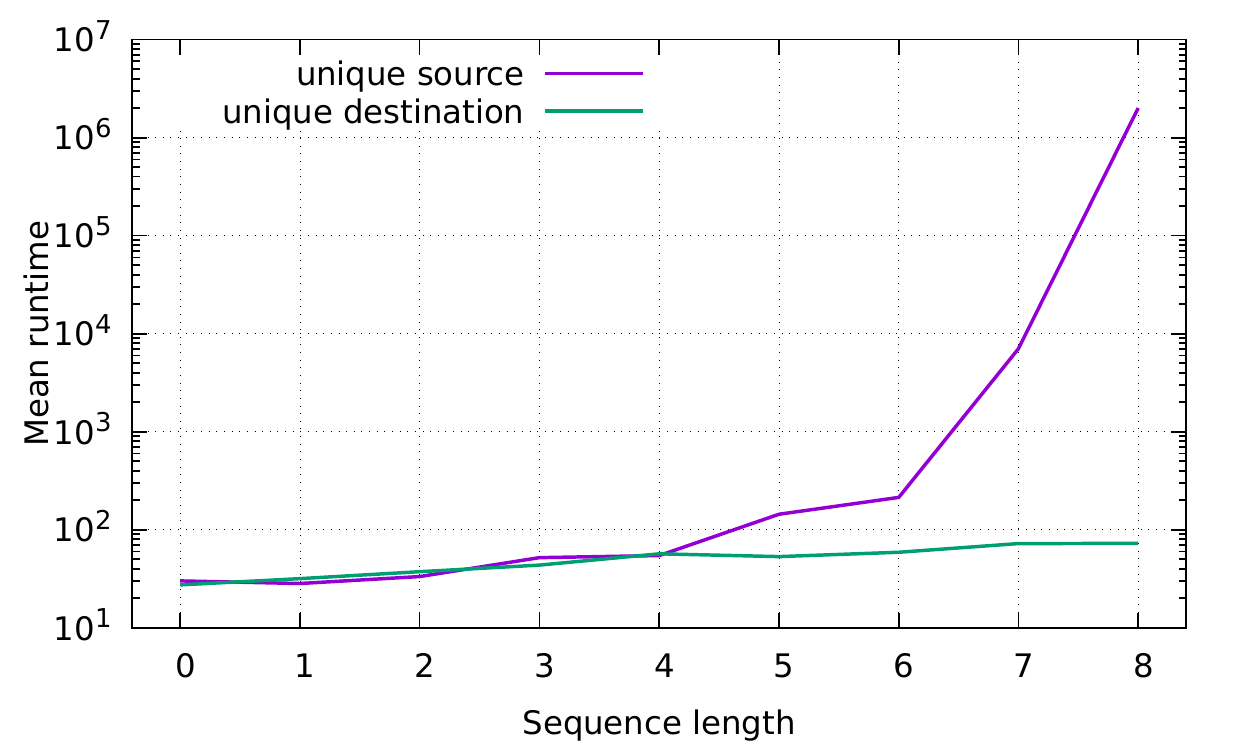}
  \caption{Mean runtime of $(1+1)$ EA on functions in classes
    $\mathcal{F}_t^\delta$ (unique destination) and
    $\mathcal{F}_t^\sigma$ (unique source) as a function of $t$
    ($n = 9$, 50 runs).}
  \label{fig:ts_aom_delta_sigma_runtime_t}
\end{figure}

Fig.~\ref{fig:ts_aom_disjoint_runtime} shows the mean runtime of
$(1+1)$ EA on TS-AOM functions with disjoint transvections and
sequence length $\floor{n/2}$ as a function of $n$. The results
suggest that the expected runtime of $(1+1)$ EA is asymptotically
slightly greater on $\mathcal{F}_{\floor{n/2}}^d$ than on OneMax.

\begin{figure}
  \centering
  \includegraphics[width=0.8\linewidth]{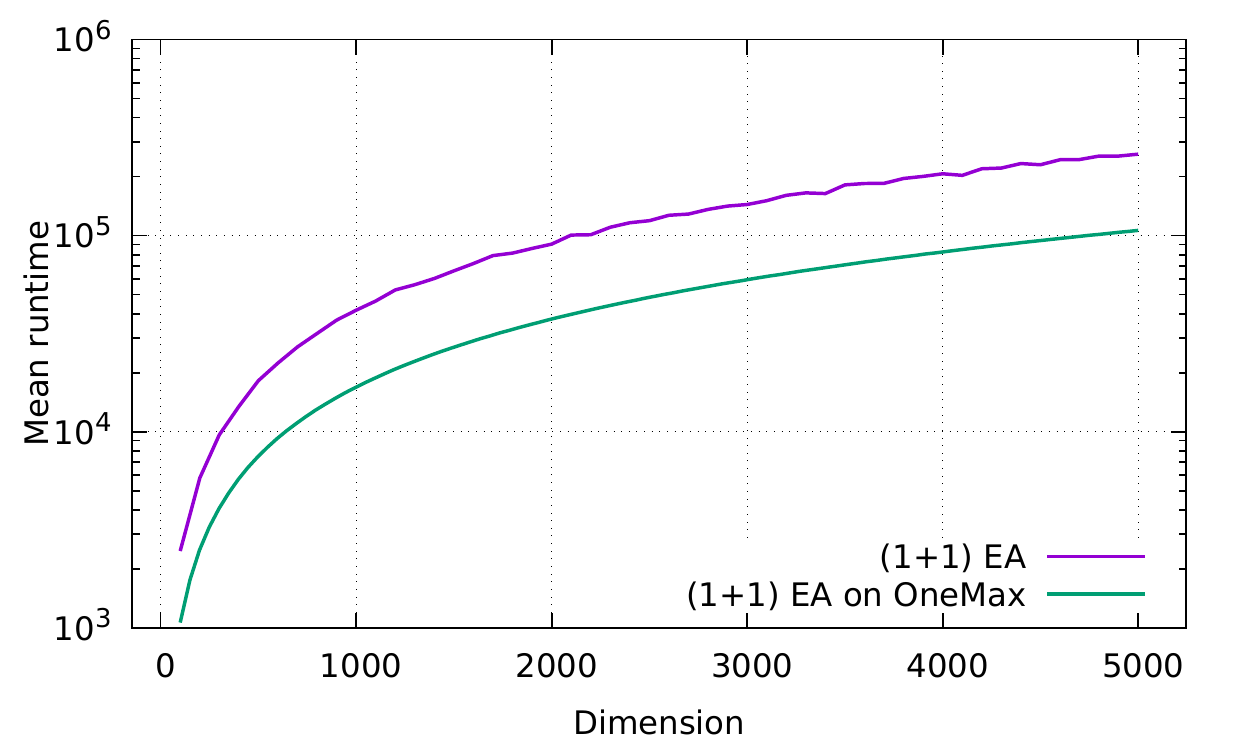}
  \caption{Mean runtime of $(1+1)$ EA on TS-AOM functions with
    disjoint transvections and sequence length $\floor{n/2}$ as a
    function of $n$ (100 runs).}
  \label{fig:ts_aom_disjoint_runtime}
\end{figure}

\section{Conclusion}
\label{sec:conclusion}

We have introduced affine OneMax functions which are test functions
for search algorithms. They are defined as compositions of OneMax and
invertible affine maps on bit vectors. They have a simple
representation and a known maximum. Tunable complexity is achieved by
expressing invertible linear maps as finite products of transvections.
The complexity is controlled by the length of transvection sequences
and their properties. Transvection sequence AOM functions with small
sequence length are of practical interest in the benchmarking of
search algorithms. We have shown by means of Fourier analysis that the
black box complexity of AOM functions is upper bounded by a high
degree polynomial. However, it can be as low as logarithmic for the
simplest AOM functions.

Many open questions remain. The gap between the lower and upper bounds
of the black box complexity in the general case is significant and
should be reduced. A rigorous analysis of the runtime of $(1+1)$ EA on
AOM functions would be of great interest. OneMax is one of the
simplest nontrivial functions used in black box optimization. It seems
legitimate to consider alternative functions to define new classes in
the same way as for AOM functions. Candidate functions include
pseudo-linear functions or LeadingOnes. However, since LeadingOnes is
not sparse, the method presented in this paper does not apply.

\bibliographystyle{plain}

\bibliography{bibliography}

\end{document}